\documentclass{article}
\usepackage{iclr2027_conference,times}

\usepackage{amsmath,amsfonts,bm}

\def\eqref#1{equation~\ref{#1}}

\def\1{\bm{1}}

\DeclareMathAlphabet{\mathsfit}{\encodingdefault}{\sfdefault}{m}{sl}
\SetMathAlphabet{\mathsfit}{bold}{\encodingdefault}{\sfdefault}{bx}{n}

\usepackage{amsmath}
\usepackage{amssymb}
\usepackage{array}
\usepackage{booktabs}
\usepackage{enumitem}
\usepackage{graphicx}
\usepackage{hyperref}
\usepackage{multirow}
\usepackage{colortbl}
\usepackage{wrapfig}
\usepackage{needspace}
\usepackage{capt-of}

\definecolor{rowband}{HTML}{EDF1F5}
\definecolor{colglobal}{HTML}{FAF1E2}
\definecolor{coloracle}{HTML}{E8E1F2}
\definecolor{colchain}{HTML}{F4F0FA}
\usepackage{tabularx}
\usepackage{xcolor}
\usepackage{url}

\definecolor{CoreBlue}{HTML}{1F5A94}
\definecolor{TodoOrange}{HTML}{B45309}
\definecolor{PlaceholderGray}{HTML}{F3F4F6}

\usepackage{amsthm}
\usepackage{tcolorbox}
\tcbuselibrary{skins,breakable}
\newtheorem{theorem}{Theorem}

\newtheorem{definition}{Definition}
\newtheorem{lemma}{Lemma}

\newtheorem{remark}{Remark}
\newtheorem*{restatethm}{Theorem}
\definecolor{gaingreen}{RGB}{25,107,36}
\definecolor{dropred}{RGB}{165,49,49}
\newcommand{\gaintag}[1]{{\color{gaingreen}\boldmath\bfseries(#1)}}
\definecolor{mProgressLM}{HTML}{427BA8}
\definecolor{mRoboDopamine}{HTML}{3A8360}
\definecolor{mRoboMeter}{HTML}{A5671C}
\definecolor{mTOPReward}{HTML}{AF5F5D}
\definecolor{mVLAC}{HTML}{2C808B}
\newcommand{\mname}[2]{\textcolor{#1}{\textbf{#2}}}
\usepackage{cancel}
\newcommand{\gone}[1]{{\color{dropred}\cancelto{\color{gaingreen}0}{#1}}}
\definecolor{NUPurple}{HTML}{4E2A84}
\definecolor{NUPurple60}{HTML}{836EAA}
\definecolor{NUPurple30}{HTML}{B6ACD1}
\definecolor{NUPurple10}{HTML}{E4E3EC}
\tcbset{nubox/.style={enhanced jigsaw, breakable,
  colback=NUPurple!4!white, colframe=NUPurple,
  arc=5pt, boxsep=5pt, left=10pt, right=10pt, top=2pt, bottom=2pt,
  boxrule=0.8pt, drop shadow=gray!50!white,
  before skip=8pt, after skip=8pt}}
\newtcolorbox{thmbox}{nubox, unbreakable, top=3pt, bottom=3pt}
\newtcolorbox{keybox}{nubox}

\makeatletter
\def\thm@space@setup{\thm@preskip=3pt \thm@postskip=3pt}
\makeatother

\newsavebox{\fittitlebox}
\newdimen\fittitlesize
\newdimen\fittitlemin
\newif\iffittitlecontinue
\newcommand{\fittitle}[1]{%
  \fittitlesize=17pt\relax
  \loop
    \setbox\fittitlebox=\vbox{\hsize=\textwidth
      \fontsize{\fittitlesize}{1.18\fittitlesize}\selectfont\scshape #1\par}%
    \ifdim\dimexpr\ht\fittitlebox+\dp\fittitlebox\relax>2.7\fittitlesize
      \ifdim\fittitlesize>\fittitlemin
        \advance\fittitlesize by -0.5pt
        \fittitlecontinuetrue
      \else
        \fittitlecontinuefalse
      \fi
    \else
      \fittitlecontinuefalse
    \fi
  \iffittitlecontinue
  \repeat
  \fontsize{\fittitlesize}{1.18\fittitlesize}\selectfont #1}

\title{\fittitle{\raisebox{-0.22\height}{\includegraphics[height=1.65em]{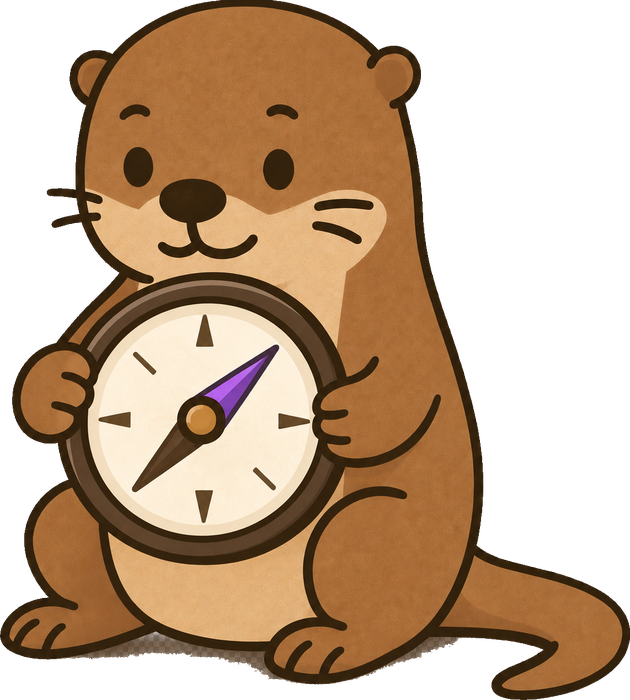}}\hspace{0.25em}\texttt{\textbf{ProgressCompass}}: Embodied Progress Reward Models Are Lost Without the Right Context}}

\author{%
Jianshu Zhang\textsuperscript{1}$^{*}$\quad
Keliang Wu\textsuperscript{1}$^{*}$\quad
Chengxuan Qian\textsuperscript{2}\quad
Xiyuan Yang\textsuperscript{3}\quad
Ce Zhang\textsuperscript{4} \\
\bfseries Ariel Tian\textsuperscript{1}\quad
Anbang Liu\textsuperscript{1}\quad
Haoran Lu\textsuperscript{1}\quad
Han Liu\textsuperscript{1} \\
\normalfont \textsuperscript{1}Northwestern\qquad
\textsuperscript{2}UCSB\qquad
\textsuperscript{3}UIUC\qquad
\textsuperscript{4}CMU\qquad
$^{*}$Equal contribution \\
\normalfont\fontfamily{pcr}\selectfont Project Page: \href{https://andyzworks.github.io/progresscompass/}{\textcolor{red}{https://andyzworks.github.io/progresscompass}}%
}

\iclrfinalcopy

\begin{document}

\maketitle

\begin{abstract}
Embodied agents now take on ever longer tasks. For long tasks, knowing only
whether a task finally succeeds or fails says little; the steps along the way
matter. Progress Reward Models (PRMs) score how far a task has come at every step, and
serve as dense rewards, verifiers and monitors. Yet in long tasks the current
frame alone often cannot tell how far the task has come, because progress
depends on what happened before. We call this problem
\emph{context-dependent progress estimation}. Existing benchmarks on progress estimation mostly focus on short tasks whose
progress can be read from the current observation, and whether PRMs can estimate
progress when context is needed remains underexplored. We therefore build
\textsc{ContextProgress-Bench}, with $24$ manipulation tasks for $120$ episodes. The benchmark covers three settings: (i) \emph{State Recall}, where information
needed for progress appeared earlier but is not in the current frame; (ii) \emph{Sequence Tracking}, where steps follow a fixed order, so progress
requires knowing which steps are done and which comes next;
and (iii) \emph{Recurrence Disambiguation}, where look-alike frames sit at very
different progress. We then run a paired diagnosis: each PRM
keeps the same input format in both runs, and in one run its instruction
integrates the right context. Results show that even for PRMs
that read the entire history would
still get lost in estimating progress. However, with the right context, the same
five models cut their progress error by \mbox{$77$--$82\%$}. This suggests that Embodied PRMs are thus not incapable in progress
estimation, but lost without the right context. We therefore propose
\texttt{ProgressCompass}, an autonomous agentic loop that reorients an existing
PRM and uses current general-purpose VLMs to supply the context the PRM needs.
Wrapped in the loop, the same frozen PRM cuts its progress error by $63\%$ and
raises its rank agreement by $76\%$. With such a compass, PRMs estimate progress
far better on longer, more complex tasks.
\end{abstract}

\vspace{2pt}
\begin{center}
  \includegraphics[width=0.86\linewidth]{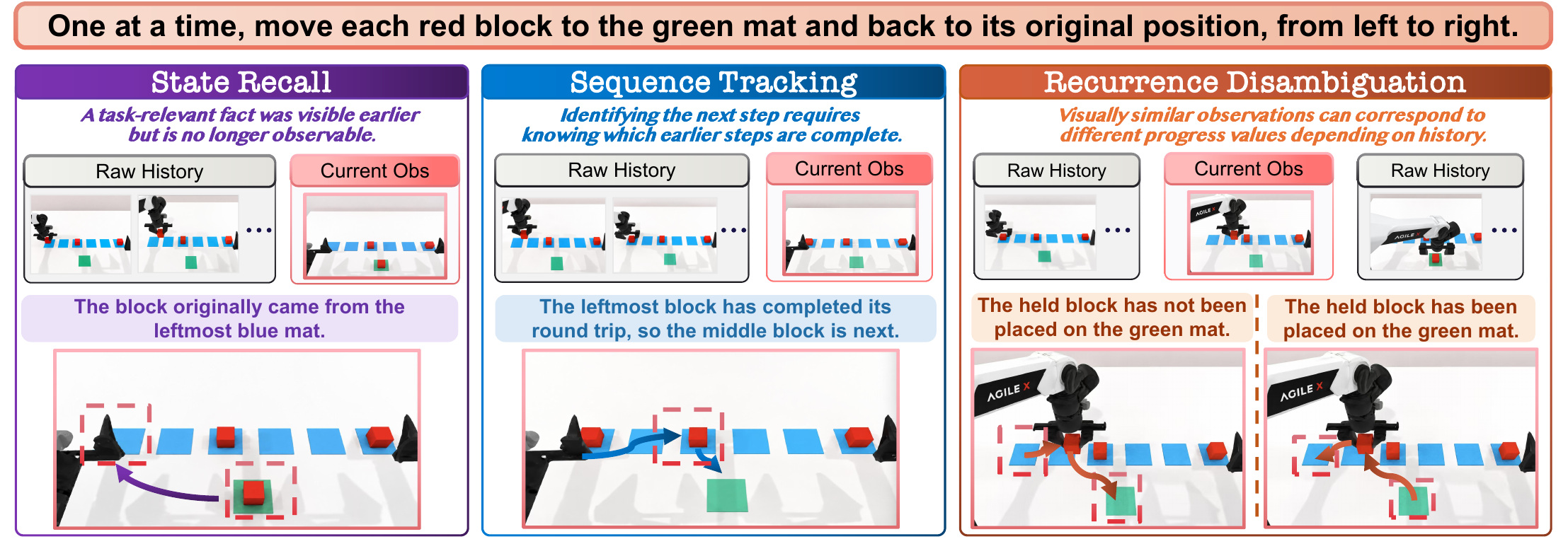}
  \vspace{-4pt}
  \captionof{figure}{\textbf{Why a Progress Reward Model (PRM) needs context.} Given
  the overall task instruction (top), the current frame alone may not suffice to
  estimate progress. (i)~\emph{State Recall}: information
  needed for progress, such as which mat the block came from, appeared earlier but
  is not in the current frame. (ii)~\emph{Sequence Tracking}: steps follow a fixed
  order, so progress needs knowing which steps are done. (iii)~\emph{Recurrence
  Disambiguation}: look-alike frames sit at different progress.}
  \label{fig:teaser}
\end{center}


\clearpage
\raggedbottom
\section{Introduction}

Embodied agents now take on ever longer tasks. For a long task, knowing only
whether the task finally succeeds or fails says little; the steps along the way
matter. A Progress Reward Model (PRM) scores how far a task has come at every
step. PRMs serve as dense rewards where the environment gives
none~\citep{ma2023vip,zhang2025rewind,fei2026srpo}, as verifiers that decide
whether a step is finished~\citep{du2023success}, and as monitors that notice
when an execution has stalled~\citep{park2026hideandseek}. Recent PRMs trained on large robot corpora score
arbitrary instructions and
trajectories~\citep{zhang2026progresslm,tan2026robodopamine,liang2026robometer,zhang2026vlac}.

These PRMs work well when the current frame shows the answer: a glass is half
poured when it looks half full. Many long tasks are not like this.
Figure~\ref{fig:teaser} shows one instruction: move each red block to the green
mat and back to its original position, one at a time, from left to right. Three
moments in this task cannot be scored from the current frame alone. When a block
is carried back, the frame shows the block but not which blue mat it came from,
so the PRM must recall a state that is no longer visible (\emph{State Recall}).
When the middle block is moving, the frame does not show whether the leftmost
block has already finished its round trip, so the PRM must track where the task
is in its fixed order (\emph{Sequence Tracking}). Just before and just after the
held block is placed, the two frames look almost the same but sit at different
progress, so the PRM must tell which moment it is in (\emph{Recurrence
Disambiguation}). In each case progress is well defined, but the current frame
does not determine it. We call this problem \emph{context-dependent progress
estimation}.

The obvious fix is to give the PRM the whole history. The history, however, is
only a record of what was observed. What the PRM needs is one specific fact that
the record establishes, such as which mat a block came from or how many round
trips are done, and which fact matters depends on the instruction and on the
current frame. We call this fact the \emph{context}.

Existing benchmarks on progress estimation, such as value-order evaluation over
robot datasets~\citep{ma2025gvl,budzianowski2025opengvl} and Progress-Bench~\citep{zhang2026progresslm},
mostly focus on short tasks whose progress can be read from the current
observation, so whether PRMs can estimate progress when context is needed remains
underexplored. We therefore build \textsc{ContextProgress-Bench}, a
high-quality benchmark of $24$ tasks, $120$ episodes and $552$ annotated subtask
intervals from RMBench~\citep{chen2026rmbench},
RoboDojo~\citep{chen2026robodojo} and LIBERO-Mem~\citep{chung2026memory}. We select every task by hand so
that it needs context, and annotate by hand the context each episode requires in
each time interval.

We then run a paired diagnosis on five PRMs. For each PRM, both runs use the same
input format, and in one run the instruction integrates the right context, so only
the context differs. Without the right context, every PRM is far off: MAE
lies between $15.6$ and $31.0$ on a $0$ to $100$ scale. Models that read or
retrieve from the entire history still get lost.
With the right context, the MAE of every PRM falls by $77\%$ to $82\%$, to
between $3.4$ and $6.7$, and the gain holds
on $93.3\%$ of paired intervals rather than on a few easy episodes.
Anchored on its own previous endpoint instead of the true one, every PRM still
improves but far less. The
errors also follow the three settings: a PRM forgets a state that has left the
frame, drifts ahead in an ordered task, or gives near-identical values to
repeated events. Embodied PRMs are thus not incapable, but lost without the right
context.

This finding tells us what to build, and it points to two kinds of models that
complement each other. A PRM scores a step well once it knows which step is
underway, but it cannot work out that step from the history. Vision-Language
Models (VLMs) are poorly calibrated estimators of progress, but they are good at
understanding a task, breaking it into steps and providing context. Each does
well what the other does poorly, and we design an agentic loop around this split.
We propose \texttt{ProgressCompass}, an autonomous agentic loop that keeps the
PRM frozen and uses current general-purpose VLMs to supply the context it needs.
An Orienter reads the frames and states the current step, the frozen PRM scores
that step, a Verifier checks whether the step is finished, and a text-only
Navigator runs the loop.

Wrapped in the loop, the same frozen PRM cuts its progress error by $63\%$ and
raises its rank agreement by $76\%$, closing $78\%$ of the gap to ground-truth
context. It stays robust when the video stops early (Early stop), does more than
asked (Extra steps), or follows an unrelated task (Mismatch). We further ablate
its components, and parallelize the loop so that no component within the loop sits idle, which cuts
wall-clock time by $65.6\%$ and keeps the cost of the added components low.

\newpage
Our contributions can be summarized as follows.
\begin{itemize}[leftmargin=1.4em,itemsep=1pt,topsep=2pt,parsep=0pt]
  \item We formulate \emph{context-dependent progress estimation} and build
  \textsc{ContextProgress-Bench}, which isolates its three settings.
  \item Through a paired diagnosis, we show that current PRMs are lost without the
  right context, even with the full history, and are reliable once the right
  context is supplied.
  \item We propose \texttt{ProgressCompass}, an autonomous agentic loop that
  supplies the right context to a frozen PRM with general-purpose VLMs, and
  estimates progress accurately and robustly.
\end{itemize}

\flushbottom
\section{Context-Dependent Progress Estimation}
\label{sec:context-dependent-progress}

\subsection{Task Formulation}
\label{sec:formulation}

Let $x$ be a task instruction. Let $O=(o_1,\ldots,o_T)$ be one execution of the task, where $T$ is the number of observations and $o_t$ is the observation at time $t$. The ground-truth progress at that time is $p_t\in[0,1]$, and the model predicts $\hat p_t$.

\textbf{We study progress annotation.} The model assigns a progress value to every observation of an execution. At time $t$, let $H_t=o_{1:t-1}$ denote the observation history before $o_t$, with $H_1=\varnothing$. What the context has to supply at $t$ comes from $x$, $H_t$, and $o_t$; later observations do not change it.

Of these inputs, the simplest Progress Reward Model (PRM) uses only the current observation, $\hat p_t=f(o_t\mid x)$. We study tasks where this cannot work, because progress depends on information that $o_t$ does not carry. We call such progress estimation \emph{\textbf{context-dependent}}.

\textbf{The right context makes progress estimable from $o_t$.}
If the current observation lacks a piece of information, the natural fix is to supply it. We call this missing piece the \emph{context} $c_t$. Without it, no estimator can recover $p_t$; with it, the current observation is enough again:
\begin{equation}
  \underbrace{f(o_t\mid x)}_{\text{without context}}\;\neq\;p_t,
  \qquad\qquad
  \underbrace{f(o_t\mid x,\,c_t)}_{\text{with context}}\;=\;p_t.
  \label{eq:with-context}
\end{equation}
The two sides differ only by $c_t$, and the gap between them is what a context-dependent task costs a model that ignores its history. Theorem~\ref{thm:main} in Section~\ref{sec:progress-compass} turns this gap into a floor on the error of every estimator that reads only $x$ and $o_t$, whatever its capacity.

\textbf{Having the history is not having the context.}
At time $t$, only the history $H_t$ can supply $c_t$, and the obvious choice is to hand the model the history itself, as if $c_t$ were $H_t$. But Equation~\ref{eq:with-context} asks for one specific piece of information, whereas $H_t$ is a raw stream of past observations that contains it somewhere. The history helps only after it has been turned into the context.

\begin{keybox}
\textbf{\textcolor{NUPurple}{Context is history interpreted for the task and the moment.}}
$H_t$ records what was observed. $c_t$ is what that record \emph{means} given the instruction $x$ and the observation $o_t$ in front of the model: which outcomes hold, how far the task has come, which occurrence is underway. Nothing in $H_t$ states these directly, and the same history yields a different $c_t$ at a different moment, so this is an interpretation rather than a retrieval:
\begin{equation}
  c_t=\phi(H_t\mid x,\,o_t),
  \label{eq:context-progress}
\end{equation}
where $\phi$ is conditioned on the instruction $x$ and the current observation $o_t$, because $x$ says what to look for and $o_t$ says what is already visible, and together they decide which reading of the history is the one the estimator needs. Context-dependent progress estimation thus comes down to $\phi$: Section~\ref{sec:taxonomy} names the three kinds of fact that $\phi$ must recover, Section~\ref{sec:diagnosis} shows that current PRMs, even when handed the entire history, do not carry out $\phi$ on their own, and Section~\ref{sec:progress-compass} assigns $\phi$ to general-purpose VLMs around a frozen $f$.
\end{keybox}

\subsection{Three Forms of Context Dependence}
\label{sec:taxonomy}

The mapping $\phi$ in Equation~\ref{eq:context-progress} turns the history into the context that instruction $x$ and the current observation $o_t$ need. Its three inputs play fixed roles: $x$ says what to look for, $o_t$ says what is already visible, and $H_t$ is where the rest has to be found. We distinguish three forms by what $o_t$ is missing; Figure~\ref{fig:teaser} shows all three arising from a single instruction. In each form, $\phi$ reads something different from the history, and the context takes a different shape. Each form also contributes its own term to the error floor of Theorem~\ref{thm:main}, so that the cost of each can be stated separately.

\textbf{State Recall: What happened before?}
Some task-relevant facts are visible only for a while. A state change happens, and later frames no longer show whether it happened or what the state was before. Here $x$ fixes which facts matter, $o_t$ tells which of them are currently hidden, and $\phi$ has to recall those from the history. The context is the set of task-relevant facts visible earlier but not in $o_t$.

\textbf{Sequence Tracking: Where are we in the sequence?}
Some tasks consist of $K$ steps that must be completed in order. The current frame shows a step being performed, but not whether the steps before it have actually been done. Here $x$ fixes the order, $H_t$ tells which steps have been completed, and $\phi$ has to judge whether the step shown in $o_t$ is the one that comes next. Because the steps are ordered, the completed ones form a prefix of the sequence, and the first step not yet completed is the one that should be in progress; we call its index the \emph{active position} $a_t$, which runs from $1$ (nothing done) to $K+1$ (all done). The context is this position together with the judgment that $o_t$ indeed shows step $a_t$, which holds exactly when every step before the one in $o_t$ is done.

\textbf{Recurrence Disambiguation: Which occurrence is this?}
Some tasks repeat the same event several times. Frames from different repetitions are visually indistinguishable, written $o_u\approx o_v$, but sit at different progress values ($p_u\neq p_v$); so can frames from different stages of one repetition, such as just before and just after the event takes place. Here the repeated event is tied to $x$, $o_t$ shows one such look-alike moment, and $\phi$ has to place it against what the history has recorded. The context is how many occurrences of the event have been completed so far, which gives the \emph{occurrence index} $r_t$ of the current one, together with whether that occurrence is already underway. The index separates look-alike frames from different repetitions; the second part separates look-alike frames within one.

\subsection{Controlled Benchmark Construction}
\label{sec:benchmark}

\textbf{Source trajectories.}
\textsc{ContextProgress-Bench} is built from robot-manipulation trajectories in RMBench~\citep{chen2026rmbench}, RoboDojo~\citep{chen2026robodojo}, and LIBERO-Mem~\citep{chung2026memory}, itself built on LIBERO~\citep{liu2023libero}. From these sources we select by hand the tasks whose progress is context-dependent in the sense of Section~\ref{sec:formulation}, that is, tasks in which the current observation alone does not determine progress and at least one form of Section~\ref{sec:taxonomy} is needed. For every episode we then annotate by hand the subtask intervals and, for each interval, the context that estimating its progress requires: the earlier outcome it depends on, its position in the ordered plan, or the occurrence of a repeated event it belongs to.

\begin{figure}[t]
  \centering
  \includegraphics[width=\linewidth]{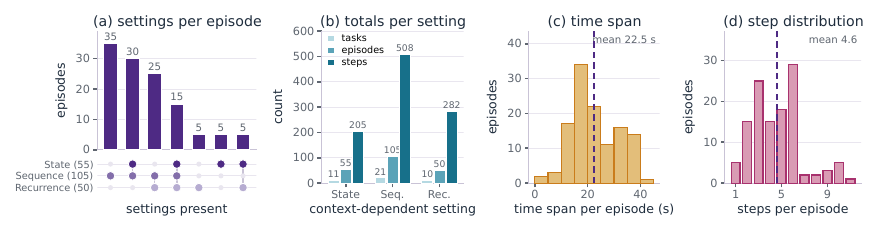}
  \caption{\textbf{Statistics of \textsc{ContextProgress-Bench}.} (a)~Episodes
  grouped by the context forms they require. (b)~Tasks, episodes and annotated
  steps per form. (c)~Episode duration. (d)~Annotated steps per episode; every
  episode needs context at least once.}
  \label{fig:benchmark-stats}
\end{figure}

\textbf{Benchmark scope.}
The benchmark contains 24 tasks, and for each task we sample five episodes at random. The resulting 120 videos contain 71,708 frames of execution. Figure~\ref{fig:benchmark-stats} summarizes its context forms, coverage, episode length, and decomposition depth.

\section{Where Progress Reward Models Get Lost}
\label{sec:diagnosis}
\label{sec:diagnostic-setup}

\textbf{Evaluation setting.}
We evaluate five Progress Reward Models (PRMs): ProgressLM-3B-RL~\citep{zhang2026progresslm},
Robo-Dopamine-GRM-3B~\citep{tan2026robodopamine},
RoboMeter-4B~\citep{liang2026robometer}, TOPReward with
Molmo2-4B~\citep{chen2026topreward,clark2026molmo2} and
VLAC-8B~\citep{zhang2026vlac}. For each model, both runs use the same input
format: \textbf{without context} gives the episode instruction, and \textbf{with
context} replaces it with the instruction of the active subtask, which integrates
the annotated context, on the interval that subtask occupies. This context is
given by hand: we write it from the annotation and set every subtask boundary
ourselves, so the model never decides where a subtask begins or ends. The two
runs are paired on all $552$ annotated subtask intervals, so they differ in $c_t$
alone. We report two metrics: (i) \textbf{MAE} is
the mean absolute error of the predicted progress, so it
asks how accurate each individual estimate is; (ii) \textbf{Spearman's $\rho$} is the
rank correlation between the predicted curve and time, so it asks whether the
estimates are ordered correctly, whatever their absolute level. A model can be
right about the ordering and wrong about the values, or the reverse, and without
context these models are often wrong about both.

\begin{table}[t]
  \caption{\textbf{The Context Gap.} \textbf{\emph{Left:}} the same frozen model with the
  same input format, first without context, then with the correct context for each
  annotated subtask. The two with-context settings differ in where that context
  puts the estimate: \emph{oracle} at the true position of its subtask,
  \emph{self-chained} at where the model ended the previous one.
  ({\color{gaingreen}\textbf{Green values}}) are the relative change
  from the same model w/o context. \textbf{\emph{Right:}} each dot is one
  subtask of one model, at its MAE without context ($x$) and with context ($y$),
  coloured as on the left; dots below the diagonal are where context helped.}
  \label{tab:diagnostic-results}
  \centering
  \begin{minipage}[t]{0.706\textwidth}
    \vspace{0pt}
    \centering
    \scriptsize
    \renewcommand{\arraystretch}{1.45}
    \setlength{\tabcolsep}{1.8pt}
    \begin{tabular}{l >{\columncolor{colglobal}}r >{\columncolor{colglobal}}r >{\columncolor{coloracle}}r >{\columncolor{coloracle}}r >{\columncolor{colchain}}r >{\columncolor{colchain}}r}
      \toprule
      \multirow{2}{*}{\textbf{Model}}
        & \multicolumn{2}{c}{\textbf{w/o context}}
        & \multicolumn{2}{c}{\textbf{w/ context: oracle}}
        & \multicolumn{2}{c}{\textbf{w/ context: self-chained}} \\
      \cmidrule(lr){2-3}\cmidrule(lr){4-5}\cmidrule(lr){6-7}
        & MAE\,$\downarrow$ & $\rho$\,$\uparrow$
        & MAE\,$\downarrow$ & $\rho$\,$\uparrow$
        & MAE\,$\downarrow$ & $\rho$\,$\uparrow$ \\
      \midrule
      \mname{mProgressLM}{ProgressLM} & 17.95 & 0.74 & 4.18\,\gaintag{$-77\%$} & 0.99\,\gaintag{$+35\%$} & 16.70\,\gaintag{$-7\%$} & 0.99\,\gaintag{$+34\%$} \\
      \mname{mRoboDopamine}{Robo-Dopamine} & 15.59 & 0.74 & 3.43\,\gaintag{$-78\%$} & 0.98\,\gaintag{$+31\%$} & 4.66\,\gaintag{$-70\%$} & 0.97\,\gaintag{$+31\%$} \\
      \mname{mRoboMeter}{RoboMeter} & 25.42 & 0.53 & 4.91\,\gaintag{$-81\%$} & 0.93\,\gaintag{$+75\%$} & 11.39\,\gaintag{$-55\%$} & 0.92\,\gaintag{$+74\%$} \\
      \mname{mTOPReward}{TOPReward} & 30.95 & 0.56 & 6.70\,\gaintag{$-78\%$} & 0.90\,\gaintag{$+62\%$} & 9.42\,\gaintag{$-70\%$} & 0.89\,\gaintag{$+59\%$} \\
      \mname{mVLAC}{VLAC} & 24.43 & 0.72 & 4.50\,\gaintag{$-82\%$} & 0.97\,\gaintag{$+35\%$} & 17.13\,\gaintag{$-30\%$} & 0.94\,\gaintag{$+30\%$} \\
      \bottomrule
    \end{tabular}
  \end{minipage}\hfill
  \begin{minipage}[t]{0.262\textwidth}
    \vspace{0pt}
    \centering
    \includegraphics[width=\linewidth]{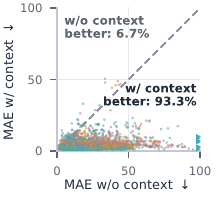}
  \end{minipage}
\end{table}
\begin{figure}[t]
  \centering
  \includegraphics[width=1.0\linewidth]{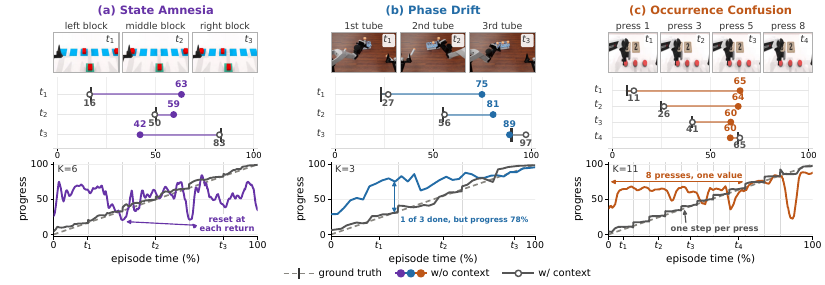}
  \caption{\textbf{Three typical errors, one episode each.} The task each panel
  runs on is (a)~move each red block onto the green mat and back to where it came
  from, from left to right; (b)~insert three tubes into a rack in a prescribed
  order; (c)~press the same button eight times.}
  \label{fig:cases}
\end{figure}

\subsection{Do Progress Reward Models Lose the Task Without Context?}
\label{sec:diagnostic-results}

They do, and by a wide margin. Without context, as
Table~\ref{tab:diagnostic-results} shows, every one of the five is far from the
true progress and orders the episode unreliably. Supplying the context
turns both around on every model at once, cutting the error by about four
fifths and bringing the ordering near perfect. The gain is not an average over
easy episodes: the right-hand panel of
Table~\ref{tab:diagnostic-results} pairs the two runs interval by interval, and
almost every pair lies below the diagonal: whether a PRM is right
turns on whether it is told where in the task it is, not on what it can see.
The oracle column is the error left once the step is known, the local term of
Theorem~\ref{thm:main}.
Anchored on its own endpoints instead of the oracle, every model still improves,
but far less, in the order of how firmly each closes a subtask.

\subsection{How the Three Forms Fail}
\label{sec:state-amnesia}

Figure~\ref{fig:cases} shows what being lost looks like in each form.
\textbf{(a)~State amnesia.} Once a state change has happened and left the frame,
everything after it is read against a scene that no longer records it. The blocks
are interchangeable and a block leaving the green mat gives no sign of where it
belongs, so from the first return onwards the curve falls back instead of
accumulating. \textbf{(b)~Phase drift.} The order of the steps is not visible in a
frame, so the model reads a step as a later one than it is, or credits an early
step as though a later one were already underway; the first insertion is complete
and the curve already sits where the second should be. \textbf{(c)~Occurrence
confusion.} Repeated actions produce near-identical frames and the model
returns near-identical values for all of them, so it cannot say which repetition
is underway: eight presses give one value, a plateau where the truth is a
staircase. In all three the model perceives the event and has nowhere to put it.

\section{\texttt{\textbf{ProgressCompass}}: Reorienting Progress Reward Models}
\label{sec:progress-compass}

\textbf{Who plays $f$.}
Section~\ref{sec:diagnosis} shows that $f$ is not where the deficit lies: given
the context $c_t$ for the moment it is looking at, a Progress Reward Model (PRM) estimates
progress reliably, which is the right-hand side of
Equation~\ref{eq:with-context}. A \textbf{PRM}
$\boldsymbol{\mathcal{P}}$ therefore plays $f$, and we keep it frozen. In our experiments $\mathcal{P}$ is
RoboMeter-4B, but any PRM that scores a clip against a step instruction can in
principle take its place, since the loop only reads the value $\mathcal{P}$
returns.

\textbf{Who plays $\phi$.}
What is missing is the map that produces that context $c_t$ of
Equation~\ref{eq:context-progress}. In Section~\ref{sec:diagnosis} we played
$\phi$ by hand: we wrote $c_t$ and set every subtask boundary. Without annotation, nobody
does either. Producing $c_t$ is task understanding rather than scoring, and PRMs
are not trained for it. Vision-Language Models are the reverse. They are not
calibrated progress estimators, but they are good at understanding a task,
breaking it into steps and providing context, so a VLM takes over the part we played by hand, and we call it
the \textbf{Orienter} $\boldsymbol{\mathcal{O}}$. A PRM scores a moment once
placed, a VLM says where the moment sits, and neither does the other's job. With
no boundaries given, the system must itself decide when each subtask ends and
orient again, so it must run as a loop.

\Needspace*{24\baselineskip}
\begin{wrapfigure}[24]{r}{0.39\linewidth}
  \vspace{-14pt}
  \setlength{\abovecaptionskip}{3pt}
  \centering
  \includegraphics[width=\linewidth]{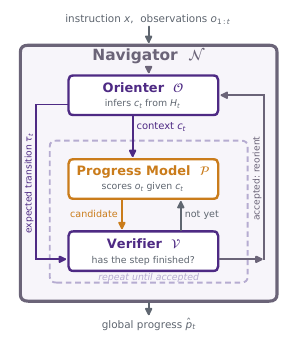}
  \caption{\textbf{\texttt{ProgressCompass}.} The Navigator $\mathcal{N}$ runs
  the loop. $\mathcal{O}$ gives the context $c_t$ to the frozen PRM
  $\mathcal{P}$ and the expected transition $\tau_t$ to $\mathcal{V}$;
  $\mathcal{P}$ proposes a completion and $\mathcal{V}$ checks it until the step
  is accepted, then $\mathcal{O}$ reorients.}
  \label{fig:progress-compass}
\end{wrapfigure}
\textbf{Run as a loop.}
$\mathcal{P}$ and $\mathcal{O}$ cover the two maps, but nothing yet makes them
run on their own. Producing $c_t$ again whenever the task moves on first requires
deciding that it has. That decision is a
yes-or-no question about the frames, which a VLM answers well even though it
cannot produce the value itself, so we hand it to a second VLM, the
\textbf{Verifier} $\boldsymbol{\mathcal{V}}$, kept apart from $\mathcal{O}$ because the
component that says what should happen should not also rule on whether it did.
Furthermore, we have $\mathcal{O}$ emit a second output beside the context, the
\textbf{expected transition} $\boldsymbol{\tau_t}$, which states what the frames should show
once the current step is done. It commits $\mathcal{O}$ to a checkable
prediction, so $\mathcal{V}$ rules on a stated outcome instead of on completion
in the abstract and the text side supplies what the pixels leave open;
Section~\ref{sec:ablation} measures what it is worth. Throughout,
$\mathcal{P}$ is given only the frames since the current step began, which is
exactly the input on which Section~\ref{sec:diagnostic-results} found it
reliable and, for a locally identifiable step, all it needs
(Definition~\ref{def:local}). Only a controller is then missing. We add a text-only one, the
\textbf{Navigator} $\boldsymbol{\mathcal{N}}$, which calls the other three and
carries the plan and what has been confirmed from one step into the next.
$\mathcal{N}$ closes a loop, orient, estimate, verify and orient again, which we
call \texttt{ProgressCompass}: the PRMs of Section~\ref{sec:diagnosis} are not
blind but lost.

\textbf{Reduced error floor by \texttt{ProgressCompass}.}
This division of labour can be stated as a bound. With the $K$ steps of the plan
and the active position $a_t$ of Section~\ref{sec:taxonomy}, the true progress is
$p_t=\big((a_t-1)+p^{\mathrm{loc}}_t\big)/K$, where
$p^{\mathrm{loc}}_t\in[0,1]$ is the progress within step $a_t$, and
\texttt{ProgressCompass} outputs
$\hat p_t=\big((k_t-1)+\hat p^{\mathrm{loc}}_t\big)/K$, where $k_t$ is the step
that $\mathcal{N}$ holds and $\hat p^{\mathrm{loc}}_t$ the estimate of
$\mathcal{P}$. The error $\varepsilon=\mathbb{E}\,|\hat p_t-p_t|$ averages over
the frames of a task, as MAE does; $\varepsilon^{f}$ is that of any current-frame
estimator $f(o_t\mid x)$ of Equation~\ref{eq:with-context}, and
$\varepsilon^{\mathrm{PC}}$ that of \texttt{ProgressCompass}. The \emph{context
floor} $\delta_{\mathrm X}$ of form $\mathrm X$ is the error that no estimate
from $x$ and $o_t$ alone can avoid on that form: for two frames that show the
same scene at different progress, half their progress gap, weighted by how often
such frames occur. The \emph{position error} $\eta^{\mathrm{PC}}$ counts how many
steps $k_t$ is off from $a_t$, at least one for a wrong position and the largest
error for a wrong plan. The \emph{local error} $\lambda$ is the error of
$\mathcal{P}$ within a step, on the frames where $k_t=a_t$.

\begin{figure}[t]
  \centering
  \includegraphics[width=\textwidth]{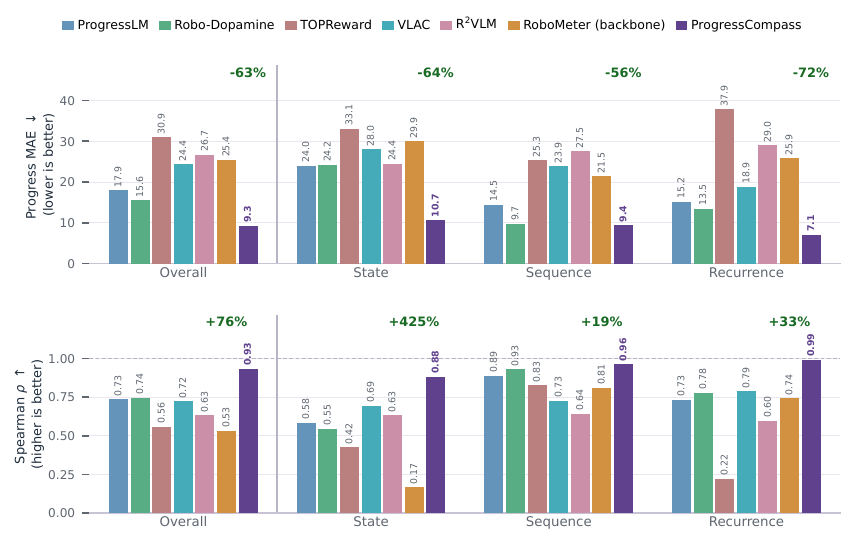}
  \caption{\textbf{Main results.} Progress MAE (top, lower is better) and
  Spearman $\rho$ (bottom, higher is better) per context form;
  {\color{gaingreen}\textbf{green}} is the change over the frozen RoboMeter.}
  \label{fig:main-bars}
\end{figure}

\begin{thmbox}
\begin{theorem}[Context floor and error of \texttt{ProgressCompass}]
\label{thm:main}
For every $K$-step task,
\begin{align*}
  \text{(no context, any $f$)}\quad
  \varepsilon^{f} \;&\ge\;
  \underbrace{{\color{dropred}\delta_{\mathrm{S}}}}_{\text{State}}
  \;+\;\underbrace{{\color{dropred}\delta_{\mathrm{Q}}}}_{\text{Sequence}}
  \;+\;\underbrace{{\color{dropred}\delta_{\mathrm{R}}}}_{\text{Recurrence}},\\[6pt]
  \text{(\texttt{ProgressCompass})}\quad
  \varepsilon^{\mathrm{PC}} \;&\le\;
  \underbrace{\gone{\delta_{\mathrm{S}}}}_{\text{outcomes}}
  \;+\;\underbrace{\gone{\delta_{\mathrm{Q}}}}_{\text{step}}
  \;+\;\underbrace{\gone{\delta_{\mathrm{R}}}}_{\text{occurrence}}
  \;+\;\underbrace{\frac{\eta^{\mathrm{PC}}}{K}}_{\text{position ($\mathcal{V}$)}}
  \;+\;\underbrace{\frac{\lambda}{K}}_{\text{local ($\mathcal{P}$)}}.
\end{align*}
The first bound holds for every such $f$. In the second bound,
a step is passed before it is complete only if $\mathcal{P}$ proposes its
completion early and $\mathcal{V}$ accepts, so $\mathcal{V}$ lets an early
proposal through with probability at most $\beta$, its false-accept rate. We put
the proof in Appendix~\ref{app:proofs}.
\end{theorem}
\end{thmbox}

\textbf{Context is irreplaceable.}
The first row is the cost of missing context. Frames that show the same scene
receive the same estimate, so each form of Section~\ref{sec:taxonomy} adds a
floor that no estimator reading only $x$ and $o_t$ can remove, whatever its
capacity. The row covers only such current-frame estimators; that the PRMs of
Section~\ref{sec:diagnosis}, which read a sampled history, are lost as well is an
empirical finding (Figure~\ref{fig:cases}). In the second row, $c_t$
states which outcomes hold, which step is underway and which occurrence it is,
so none of the three floors remains, and errors in $c_t$ are charged to the two
terms that follow. In our runs, $c_t$ is the subtask instruction from
$\mathcal{O}$, $\mathcal{P}$ is RoboMeter-4B on the frames of the current step,
and $k_t$ is the step $\mathcal{N}$ holds. $\mathcal{P}$ pays the local term, and
$\mathcal{V}$ keeps the position term small: a premature advance
needs a false accept, and because $\mathcal{N}$ records an outcome only on
acceptance, it does not corrupt later steps; a missed completion only delays the
position (Lemmas~\ref{lem:accumulation} and~\ref{lem:verification}).
\texttt{ProgressCompass} is therefore more accurate than every current-frame
estimator whenever
$(\eta^{\mathrm{PC}}+\lambda)/K<\delta_{\mathrm S}+\delta_{\mathrm Q}+\delta_{\mathrm R}$,
a condition rather than a guarantee
(Appendix~\ref{app:assumptions}).

\begin{figure}[t]
  \centering
  \includegraphics[width=\linewidth]{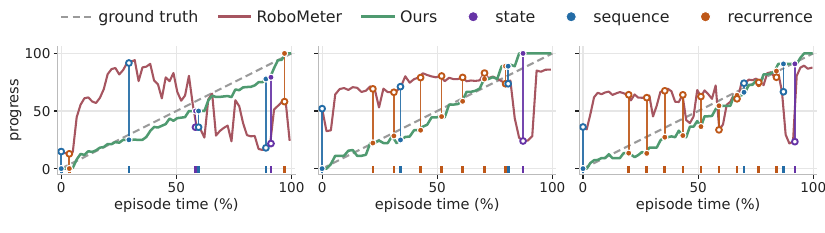}
  \vspace{-14pt}
  \caption{\textbf{Progress curves on three episodes.} Ticks mark moments that
  need context; vertical lines join ours (filled) and RoboMeter (hollow).}
  \label{fig:curves}

\end{figure}

\section{Results}
\label{sec:experiments}

\textbf{Experimental setup.}
We implement \texttt{ProgressCompass} with
Qwen3.5-27B~\citep{qwen2026qwen35} as the Orienter, Qwen3.5-9B as the Verifier and, without frames, as the Navigator, and
RoboMeter-4B, lightweight and not reliant on careful prompting, as the frozen Progress Reward Model (PRM), decoding at temperature $0$. We
compare against the five PRMs of Section~\ref{sec:diagnosis}, each
reading the full episode and instruction, and against
R$^2$VLM~\citep{zhang2026r2vlm}, which conditions on retrieved context and is the
closest existing method. RoboMeter-4B is both a baseline and our backbone, so
every gain uses the same frozen weights.

\subsection{Overall Performance}
\label{sec:main-results}

As shown in Figure~\ref{fig:main-bars}, wrapping the frozen RoboMeter-4B in
\texttt{ProgressCompass} more than halves its progress error, from $25.4$ to
$9.3$, and lifts its rank agreement from $0.53$ to $0.93$. No weight changes and the inputs are the same, so the whole gain
comes from the context. It is also the best method on every
context form, ahead of every frozen model and of R$^2$VLM, and recovers $78\%$
of the deficit that oracle context would close. Transition timing improves with it: boundary
error falls by two thirds, and the share of annotated boundaries matched within
$5\%$ of episode duration roughly triples. Figure~\ref{fig:curves} shows the same gap on single episodes.






\subsection{When Instruction and Execution Do Not Match}
\label{sec:context-interventions}

\Needspace*{16\baselineskip}
\begin{wrapfigure}[15]{r}{0.38\linewidth}
  \vspace{-12pt}
  \setlength{\abovecaptionskip}{2pt}
  \centering
  \includegraphics[width=\linewidth]{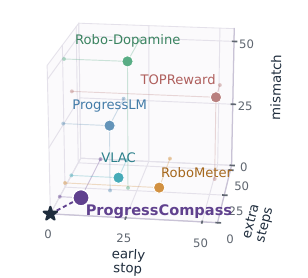}
  \caption{\textbf{Negatives.} Mean deviation of each method from the correct
  progress on Early stop, Extra steps and Mismatch; lower is better.}
  \label{fig:negatives}
\end{wrapfigure}
We build $15$ held-out negatives by breaking the correspondence between an
instruction and its execution in three directions: the execution does less than
asked, more than asked, or something unrelated. Each has its own correct
behaviour, and we score the whole curve against it (Figure~\ref{fig:negatives}).

\textbf{Early stop} keeps the instruction and cuts the video short, so the
correct value is the progress performed. Every action shown is correct, so a
model that reads the scene sees a finished task: TOPReward
deviates by $50.3$ and RoboMeter by $32.7$, against $7.3$ for
\texttt{ProgressCompass}.

\textbf{Extra steps} keeps the video and deletes steps from the instruction, so the video does more than was asked and the correct terminal value
is $100$. \texttt{ProgressCompass} deviates by $15.8$ against $24.0$ to $40.2$,
and alone reaches $100$.

\textbf{Mismatch} substitutes a structurally valid instruction from
an unrelated domain, so the correct curve is zero everywhere.
\texttt{ProgressCompass} never exceeds zero. VLAC also scores $0.0$,
but not by rejecting: with the correct instruction it ends at $-63.8$. ProgressLM, Robo-Dopamine and TOPReward deviate by $25.4$,
$49.1$ and $32.5$.

\Needspace*{6\baselineskip}
\subsection{What the Expected Transition Buys}
\label{sec:ablation}

\Needspace*{14\baselineskip}
\begin{wrapfigure}[13]{r}{0.40\linewidth}
  \vspace{-12pt}
  \setlength{\abovecaptionskip}{2pt}
  \centering
  \includegraphics[width=\linewidth]{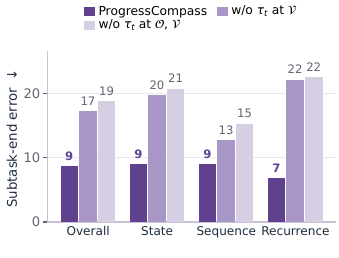}
  \caption{\textbf{Ablating $\tau_t$.} Error at the end of each subtask;
  lower is better.}
  \label{fig:ablation}
\end{wrapfigure}
We remove the expected transition two ways: \emph{w/o $\tau_t$ at $\mathcal{V}$}
keeps it everywhere else ($\mathcal{O}$ emits it and $\mathcal{P}$ receives it)
but $\mathcal{V}$ sees only the frames and the subtask sentence, and
\emph{w/o $\tau_t$ at $\mathcal{O}$ and $\mathcal{V}$} drops it from
$\mathcal{O}$'s output altogether. We score the error at the end of each
annotated subtask $k$, where a correct curve reads $k/K\cdot 100$ and
$\mathcal{V}$ uses $\tau_t$ to accept a completion. Removing $\tau_t$ at
$\mathcal{V}$ doubles this error from $8.6$ to $17.2$, and removing it at
$\mathcal{O}$ as well raises it to $18.8$ (Figure~\ref{fig:ablation}). The order
holds on every form: the error more than doubles on State, triples on
Recurrence, and rises least on Sequence, whose steps end in a visible action.
Progress MAE rises from $9.3$ to $17.5$ and $18.3$. Verification changes, not
planning: $\mathcal{O}$ proposes the same number of steps in
all three runs, while episodes that stall with a subtask unverified rise from $7$
to $40$ and $45$. Removing $\tau_t$ thus keeps the plan and enlarges the position
term of Theorem~\ref{thm:main}.

\subsection{Parallelizing the Loop Across Episodes}
\label{sec:efficiency-qualitative}

The loop is sequential within an episode, since $\mathcal{O}$ proposes the next
step only after $\mathcal{V}$ has checked the current one, but episodes share
nothing. We therefore run up to five episodes at once, and a dependency-aware scheduler sends each episode's next
ready call to the model it needs, so that a model idle for one episode serves
another. The order within every episode is
unchanged: time per episode falls from $114.7$ to $39.5$ s, with MAE and $\rho$ unchanged within paired bootstrap intervals.

\section{Related Work}
\label{sec:related-work}

\textbf{Embodied progress and process reward models.}
Current progress models are built in three main ways~\citep{zhang2026prmsurvey}:
frozen foundation models score progress
zero-shot~\citep{du2023success,sontakke2023roboclip,ma2025gvl,budzianowski2025opengvl,chen2026topreward};
temporal and relative supervision derives it from frame order or comparisons
within
demonstrations~\citep{dwibedi2019temporal,ma2023vip,ma2023liv,donahue2024progress,huang2024diffusionreward};
and models trained on progress targets add task
stages~\citep{hung2025victor,chen2026sarm,zhang2025rewind} and now score
arbitrary
trajectories~\citep{zhang2026progresslm,tan2026robodopamine,liang2026robometer,zhang2026vlac,zhang2026r2vlm}.
Their evaluations mostly use tasks whose progress can be read from the current
observation, so context-dependent progress estimation remains underexplored.

\textbf{Agentic context management.}
A long-horizon agent cannot keep all it has seen in its prompt, so it manages
what it carries forward: language agents curate their working
context~\citep{yao2023react,zhang2026memact,yi2026contextmgmt}, web agents
reflect, roll back and aggregate evidence~\citep{hu2025webcot,wang2026webaggregator},
and long-video models keep an explicit
memory~\citep{song2024moviechat,he2024malmm,zhang2026streamscout}. Embodied agents
ground instructions into plans, revise them from
feedback~\citep{ichter2023saycan,huang2023innermonologue,singh2023progprompt,liang2023codeaspolicies,huang2023voxposer},
and remember the scene~\citep{pashevich2021episodic,liu2025dynamem}. Embodied
progress estimation needs the same: over a long task, a PRM that reads the raw
history gets lost. We bring context management to
a frozen PRM, which needs one specific fact at each step.

\section{Conclusion}
\label{sec:conclusion}

We formulate \emph{context-dependent progress estimation} and build \textsc{ContextProgress-Bench}
to isolate its three settings. A paired diagnosis shows that current Progress
Reward Models (PRMs) are not blind but lost: without the right context they stay
far from the true progress, even when they read the whole history, and each cuts
its error by $77\%$ to $82\%$ once the context is given. \texttt{ProgressCompass}
supplies that context to a frozen PRM with general-purpose VLMs, without
annotation or training. It more than halves the error of its backbone and stays
robust when the execution does less than asked, more than asked, or something
unrelated.

\clearpage

\bibliography{iclr2027_conference}
\bibliographystyle{iclr2027_conference}

\appendix
\newpage
\section{Theoretical Analysis}
\label{app:theory}

\subsection{Setup}
\label{app:setup}
\label{app:reorientation-bound}

\textbf{Step form.}
Let $x$ have $K$ ordered steps $e_1,\ldots,e_K$. At time $t$, the active position
$a_t\in\{1,\ldots,K+1\}$ is one plus the number of completed steps, $b_k$ is the
time at which $e_k$ becomes active, and $p^{\mathrm{loc}}_t\in[0,1]$ is the
progress within the active step ($0$ once all steps are done). We analyse
\begin{equation}
  p_t=\frac{(a_t-1)+p^{\mathrm{loc}}_t}{K},
  \qquad
  \hat p_t=\frac{(k_t-1)+\hat p^{\mathrm{loc}}_t}{K},
  \label{eq:gt-progress}
\end{equation}
where $k_t$ is the step $\mathcal{N}$ holds and $\hat p^{\mathrm{loc}}_t$ the
estimate of $\mathcal{P}$; the second expression is how $\mathcal{N}$ reports
progress, on a $0$ to $100$ scale in the experiments.
The step form is an analysis device. If an annotation defines progress
differently, both bounds of Theorem~\ref{thm:main} move by at most the mean
absolute difference between the two definitions. Two frames that show the same
scene and differ only in details that carry no progress count as the same
observation ($o_u\approx o_v$). An error is
$\varepsilon=\mathbb{E}\,|\hat p_t-p_t|$ over the executions and frames of a
task, as MAE estimates it.

\begin{definition}[Locally identifiable step]
\label{def:local}
Each step $e_k$ has a self-contained description $q_k$ (object, source, target,
completion condition) such that for $t\in[b_k,b_{k+1})$ the pair
$(q_k,\,o_{b_k:t})$ determines $p^{\mathrm{loc}}_t$.
\end{definition}
This is why $\mathcal{P}$ receives only the frames since the current step began;
Theorem~\ref{thm:main} does not assume it. Under Definition~\ref{def:local}, the
earlier history affects $p_t$ only through
$c^\star_t=(a_t,q_{a_t},b_{a_t})$, a sufficient context of constant size.

\subsection{Conventions of the Second Bound}
\label{app:assumptions}

No component is assumed accurate; every term is defined, not assumed small. The
bound compares $k_t$ with $a_t$, which requires the plan of $\mathcal{O}$ to
match the annotated steps; every frame of an episode whose plan does not match is
charged the largest position error $K$. If $k_t=a_t$ but $c_t$ is wrong, the
error falls in $\lambda$. For a step $e_k$, $\alpha_k$ is the probability that
$\mathcal{P}$ proposes its completion early while $\mathcal{N}$ holds $e_k$, and
$\beta_k$ the probability that $\mathcal{V}$ accepts such a proposal given one is
made; $\alpha=\max_k\alpha_k$, $\beta=\max_k\beta_k$, and $\gamma$ is the rate at
which $\mathcal{V}$ rejects a true completion. By the chain rule the rate of a
premature advance is $\alpha_k\beta_k$, with no independence assumed.

\subsection{Context Floors and Lemmas}
\label{app:lemmas}

\begin{definition}[Context floors]
\label{def:gaps}
For a scene $o$, let $m(o)=\min_{v\in[0,1]}\mathbb{E}[\,|v-p_t|\mid o_t=o\,]$.
Label each context-dependent frame with one form, and let
$\delta_{\mathrm X}=\mathbb{E}[m(o_t)\,\mathbf 1[t\text{ is labelled }\mathrm X]]$
for $\mathrm X\in\{\mathrm S,\mathrm Q,\mathrm R\}$.
\end{definition}
If the frames showing $o$ split into two equally frequent groups whose progress
differs by $\delta$, then $m(o)=\delta/2$; $m(o)=0$ whenever $o$ determines
progress. Unlabelled frames are left out, so the floors are conservative.

\begin{lemma}[Context floor]
\label{lem:context-gap}
(a) Any estimator with the same output on two frames of the same scene whose
progress differs by $\delta$ has mean error at least $\delta/2$ on them.
(b) Every estimator whose output depends only on $x$, $o_t$, and randomness
independent of the execution has
$\varepsilon\ge\delta_{\mathrm S}+\delta_{\mathrm Q}+\delta_{\mathrm R}$.
\end{lemma}
\begin{proof}
(a) For a common output $v$, $|v-p^{(i)}|+|v-p^{(j)}|\ge\delta$.
(b) Given $o_t=o$, the output is independent of $p_t$, so its conditional error is
at least $m(o)$; averaging over $o_t$ and using $m\ge0$ gives the bound.
\end{proof}

\begin{lemma}[History gap]
\label{lem:history}
Let an estimator receive $B$ frames sampled uniformly and independently from
$o_{1:t}$, and let two executions differ only in a window of $L$ frames. With
probability $(1-L/t)^B\ge1-BL/t$ no sampled frame falls in the window, and
Lemma~\ref{lem:context-gap}(a) applies.
\end{lemma}
\begin{proof}
Each sample misses the window with probability $1-L/t$; Bernoulli's inequality
gives the bound, and without such a frame both inputs are identical.
\end{proof}

\begin{lemma}[Accumulation without verification]
\label{lem:accumulation}
If $k_t$ advances whenever $\mathcal{P}$ proposes a completion, with no
verification, then (a) while $k_t>a_t$,
$\hat p_t-p_t\ge(1-p^{\mathrm{loc}}_t)/K$ until $e_{a_t}$ completes; (b) a
recorded false completion corrupts every later context that refers to it; and
(c) with independent premature advances of rate $\alpha$, an episode contains one
with probability $1-(1-\alpha)^K$.
\end{lemma}
\begin{proof}
(a) With $k_t\ge a_t+1$, Equation~\ref{eq:gt-progress} gives
$\hat p_t-p_t\ge(1+\hat p^{\mathrm{loc}}_t-p^{\mathrm{loc}}_t)/K$.
(b) Later contexts are produced from the record. (c) is the complement of no
premature advance in $K$ steps.
\end{proof}

\begin{lemma}[Verification]
\label{lem:verification}
In \texttt{ProgressCompass}, for a step $e_k$ of an episode whose plan matches:
(a) $\mathcal{N}$ passes $e_k$ early only if $\mathcal{P}$ proposes early and
$\mathcal{V}$ accepts, with probability $\alpha_k\beta_k\le\alpha\beta$ against
$\alpha_k$ without $\mathcal{V}$; (b) the record holds a false completion only in
that event; (c) a rejected true completion keeps the position one step behind
and the record unchanged; (d) after a premature advance the position realigns
when $e_k$ truly completes, unless another false accept occurs first.
\end{lemma}
\begin{proof}
$\mathcal{N}$ advances and records an outcome only after $\mathcal{V}$ accepts,
which gives (a)--(c) by the chain rule. For (d), while $e_k$ is incomplete
$a_t=k$ and $k_t=k+1$, and $a_t$ becomes $k+1$ when $e_k$ completes.
\end{proof}

\subsection{Proof of Theorem~\ref{thm:main}}
\label{app:proofs}

\begin{restatethm}[Theorem~\ref{thm:main}, with its terms spelled out]
Fix a task with $K$ steps. (i) Every estimator $f$ whose output depends only on
$x$, $o_t$, and independent randomness satisfies
$\varepsilon^{f}\ge\delta_{\mathrm S}+\delta_{\mathrm Q}+\delta_{\mathrm R}$.
(ii) $\varepsilon^{\mathrm{PC}}\le(\eta^{\mathrm{PC}}+\lambda)/K$, where
$\eta^{\mathrm{PC}}=\mathbb{E}[d_t]$ with
$d_t=(|k_t-a_t|+1)\,\mathbf 1[k_t\ne a_t]$ if the plan matches and $d_t=K$
otherwise, and
$\lambda=\mathbb{E}[\,|\hat p^{\mathrm{loc}}_t-p^{\mathrm{loc}}_t|\mid\text{the plan matches and }k_t=a_t]$.
(iii) A step is passed early with probability at most $\alpha\beta$, against
$\alpha$ without $\mathcal{V}$, and a rejected true completion only delays the
position.
\end{restatethm}

\begin{proof}
(i) is Lemma~\ref{lem:context-gap}(b), and (iii) is
Lemma~\ref{lem:verification}. For (ii), on an episode whose plan matches,
Equation~\ref{eq:gt-progress} gives
$\hat p_t-p_t=\big((k_t-a_t)+(\hat p^{\mathrm{loc}}_t-p^{\mathrm{loc}}_t)\big)/K$.
If $k_t=a_t$ the error is $|\hat p^{\mathrm{loc}}_t-p^{\mathrm{loc}}_t|/K$; if
$k_t\ne a_t$ it is at most $(|k_t-a_t|+1)/K=d_t/K$; and if the plan does not
match it is at most $1=d_t/K$. Hence at every frame
$|\hat p_t-p_t|\le d_t/K+|\hat p^{\mathrm{loc}}_t-p^{\mathrm{loc}}_t|\,
\mathbf 1[\text{the plan matches and }k_t=a_t]/K$, and taking expectations gives
$\varepsilon^{\mathrm{PC}}\le(\eta^{\mathrm{PC}}+\lambda)/K$.
\end{proof}

\begin{remark}[Weights and scope]
\label{rem:weights}
\label{rem:scope}
With step weights $w_\ell$ summing to one,
$\hat p_t=\sum_{\ell<k_t}w_\ell+w_{k_t}\hat p^{\mathrm{loc}}_t$, and every $1/K$
in (ii) becomes $\max_\ell w_\ell$; we use $w_\ell=1/K$ throughout.
The first bound covers current-frame estimators and, with the probability of
Lemma~\ref{lem:history}, estimators that read a fixed budget of sampled frames.
It does not cover a model that reads the whole history; whether such models use
it is the empirical question of Section~\ref{sec:diagnosis}. The theorem gives a
condition under which \texttt{ProgressCompass} is more accurate, not a guarantee.
\end{remark}

\section{\texttt{ProgressCompass} Implementation}
\label{app:prompts}
\label{app:implementation}

\textbf{Models and decoding.}
The Orienter $\mathcal{O}$ is Qwen3.5-27B, and the Verifier $\mathcal{V}$ and the
Navigator $\mathcal{N}$ are Qwen3.5-9B. All three decode at temperature $0$ with
thinking disabled and return JSON under a fixed schema. The frozen PRM
$\mathcal{P}$ is RoboMeter-4B. The episode is sampled every $10$ frames, with at
most $128$ frames.

\textbf{Navigator.}
$\mathcal{N}$ never sees a frame. It reads the instruction and a text state that
holds the plan, whether each step is satisfied, the verified memory, the current
position, and the step in flight, and it makes exactly one call per turn by fixed
rules: while a step is in flight it asks $\mathcal{P}$ and $\mathcal{V}$ to
resolve it; when every step is satisfied it asks $\mathcal{O}$ once to review the
plan for an omitted step, and then ends; when the video ends it ends; otherwise it
asks $\mathcal{O}$ for the next step.

\textbf{Orienter.}
$\mathcal{O}$ sees the current frame and no later frame, together with the
instruction, the plan, and the verified memory. On its first call it lists the
visible task objects with their counts and builds the plan: an ordered list of
the outcomes the instruction requires, with one step per occurrence when an
action repeats and a visually checkable criterion for each step. On later calls it
may add, correct, remove, or reorder open steps, while satisfied steps stay fixed.
It then describes the next open step: a self-contained subtask sentence, the
current state before it, the expected transition $\tau_t$, the state after it,
and a hint on whether completion is a lasting state, a brief interaction, or an
observation. Every step must be required by the instruction; the frame grounds
its objects but does not add steps.

\textbf{PRM.}
$\mathcal{P}$ scores the frames since the current step began against the subtask
sentence. From its curve we take a completion candidate: a high-score frame, the
peak before a sustained drop, or the end of the video.

\textbf{Verifier.}
$\mathcal{V}$ sees three frames in order, at the start of the step, during it, and
at the candidate, together with the subtask sentence, $\tau_t$, the predicted
state after the step, and the verified memory. It first records what each frame
shows and the visible change, with counts when the step moves one of several
identical objects, and only then decides whether the step was newly carried out
and finished. The descriptions of $\mathcal{O}$ are references, and the frames
decide every disagreement; an outcome already present at the start does not
count, and uncertain evidence is a rejection. On acceptance, its description of
the candidate frame enters the verified memory and $\mathcal{N}$ advances the
position; on rejection, the step stays in flight, with at most eight
verifications per step.

\section{Inference Details of the Compared Models}
\label{app:experimental-details}
\label{app:baseline-inference}

We run every compared model with its own input format and inference procedure, and map its output to a $0$ to $100$ progress
scale. Without context, a model scores the whole episode under the episode
instruction; with context, it scores each annotated subtask clip under the
instruction of that subtask. The settings of each model are as follows.

\textbf{ProgressLM.}
We run ProgressLM~\citep{zhang2026progresslm} with \texttt{n\_demo=5}. Without context, we set \texttt{max\_frames=0} and sample RMBench at 7.5 FPS and RoboDojo and LIBERO-Mem at 2.5 FPS; with context, we sample RMBench at 7.5 FPS and RoboDojo and LIBERO-Mem at 3.0 FPS. We parse the value enclosed by \texttt{<score>}, clip it to $[0,1]$, and multiply it by 100.

\textbf{Robo-Dopamine.}
We run Robo-Dopamine~\citep{tan2026robodopamine}. The last frame of the input video serves as the required goal image. We use \texttt{frame\_interval=5} without context and \texttt{frame\_interval=10} with context, both with \texttt{batch\_size=10}. We run the released incremental, forward, and backward modes separately, apply the mode-specific official post-processing, and average the three resulting curves.

\textbf{RoboMeter.}
We run RoboMeter~\citep{liang2026robometer} in \texttt{frame\_steps} mode with \texttt{prefix\_sample\_frames=8}. Without context, RMBench is evaluated at every original frame and RoboDojo and LIBERO-Mem at 2.5 FPS; with context, all three sources are sampled at 3.0 FPS. The predicted reward at each sampled prefix is stored on a 0 to 100 progress scale.

\textbf{TOPReward.}
We run TOPReward~\citep{chen2026topreward} with Molmo2-4B as its backend~\citep{clark2026molmo2}, with \texttt{num\_samples=48} and \texttt{long\_side=0}, and \texttt{max\_frames=96} without context and \texttt{max\_frames=48} with context. Both use mean token-log-probability reduction, with video-description generation and the chat template disabled. For each video, the raw instruction rewards over sampled prefixes are min-max normalized and multiplied by 100.

\textbf{VLAC.}
We run VLAC~\citep{zhang2026vlac} in critic mode with \texttt{compress\_fps=5}, \texttt{batch\_num=5}, \texttt{pair\_skip=5}, \texttt{temperature=0.5}, \texttt{top\_k=1}, and \texttt{think=false}; with context, we additionally set \texttt{in\_context\_done=false} and \texttt{done\_threshold=0.9}. The released preprocessing resizes input images to $448\times448$. When sampling omits the original terminal frame, the runner appends that frame and its terminal prediction to the saved curve.

\end{document}